\documentclass[10pt]{article} 
\usepackage[preprint]{tmlr}

\usepackage{hyperref}
\usepackage{url}

\title{PAC-Bayesian Reconstruction Guarantees for \\Time Series Variational Autoencoders}

\author{\name Chlo\'e Hashimoto-Cullen \email hashimoto@lpsm.paris \\
      \addr LPSM,
      Sorbonne Universit\'e
      \AND
      \name Ghislain Agoua\footnote{Work done whilst at EDF R\&D/} \email ghislain.agoua@gmail.com
      \AND
      \name Benjamin Guedj \email b.guedj@ucl.ac.uk\\
      \addr University College London and Inria\\
      \AND
      \name Sylvain Le Corff \email sylvain.lecorff@sorbonne-universite.fr\\
      \addr LPSM,
      Sorbonne Universit\'e}

\def\month{MM}  
\def\year{YYYY} 
\def\openreview{\url{https://openreview.net/forum?id=XXXX}} 

\usepackage{graphicx}
\usepackage[utf8]{inputenc} 
\usepackage[T1]{fontenc}    
\usepackage{url}            
\usepackage{booktabs}       
\usepackage{multirow}       
\usepackage{amsfonts}       
\usepackage{nicefrac}       
\usepackage{microtype}      
\usepackage{xcolor}         
\usepackage{wrapfig}        
\usepackage{array}          
\usepackage{caption}        
\usepackage{enumitem}       
\usepackage{changepage}     
\usepackage{xspace}
\usepackage{datatool}

\usepackage{algorithm}      
\usepackage{algpseudocode}  

\usepackage{amsthm}         
\usepackage{dsfont}         
\usepackage{amsmath}        
\usepackage{amssymb}        %
\usepackage{mathtools}      %

\usepackage{url}
\usepackage{hyperref}       
\usepackage{cleveref}       

\newtheorem{theorem}{Theorem}
\newtheorem{lemma}[theorem]{Lemma}
\newtheorem{proposition}[theorem]{Proposition}
\newtheorem{corollary}[theorem]{Corollary}

\usepackage{natbib}         

\newcommand{\eqsp}{\,}
\newcommand{\rmd}{\mathrm{d}}

\newcommand{\E}{\mathbb{E}}

\newcommand{\R}{\mathbb{R}}
\newcommand{\N}{\mathbb{N}}
\newcommand{\mcn}{\mathcal{N}}

\newcommand{\step}{t}
\newcommand{\datadim}{T}

\newcommand{\timerange}{1:T}
\newcommand{\inputindex}{t}
\newcommand{\datasamples}{n}
\newcommand{\samplesindex}{i}

\newcommand{\dataset}{\mathsf{S}}
\newcommand{\covardim}{d}
\newcommand{\datagendis}{\mathcal{D}}
\newcommand{\data}{\mathbf{x}}
\newcommand{\latent}{\mathbf{z}}
\newcommand{\datapt}{x_\inputindex}
\newcommand{\latentpt}{z_\inputindex}
\newcommand{\latentspace}{\mathsf{Z}^\datadim}

\newcommand{\latentspacept}{\mathsf{Z}}

\newcommand{\risk}{\mathcal{R}}
\newcommand{\kl}{\operatorname{KL}}

\newcommand{\paramenc}{\phi}
\newcommand{\paramdec}{\theta}
\newcommand{\encdistrib}{q}
\newcommand{\decdistrib}{p}
\newcommand{\vae}{\mathrm{VAE}}
\newcommand{\decfunc}{g_{\paramdec}}

\newcommand{\seriestep}{x}
\newcommand{\series}{\mathbf{\seriestep}}
\newcommand{\serie}{\series^\samplesindex}

\newcommand{\lookback}{\mathsf{L}}
\newcommand{\horizon}{\mathsf{H}}
\newcommand{\dist}{\mathrm{d}}

\algnewcommand\algorithmicinput{\textbf{Input:}}
\algnewcommand\Input{\item[\algorithmicinput]}
\algnewcommand\algorithmicoutput{\textbf{Output:}}
\algnewcommand\Output{\item[\algorithmicoutput]}

\newcommand{\vd}[1]{q_{#1}^{\paramenc}} 
\newcommand{\parvar}{\paramenc} 
\newcounter{hypH}
\newenvironment{hypH}{\refstepcounter{hypH}\begin{itemize}
\item[{\bf H\arabic{hypH}}]}{\end{itemize}}

\usepackage{csvsimple}

\begin{document}

\maketitle

\begin{abstract}
Forecasting time series accurately is critical for applications with complex data ranging from energy systems to healthcare and finance. Among current state of the art models, generative latent variable models are increasingly implemented; yet principled generalisation guarantees for modern latent variable models remain limited. In particular, while Variational AutoEncoders are widely used for sequential data, their theoretical analysis is largely restricted to i.i.d. settings. In this work, we develop a PAC-Bayesian framework for latent variables models applied to time series. Building on reconstruction-based bounds, we extend PAC-Bayesian guarantees to Markovian latent structures, capturing temporal dependencies through a sequential generative process. These guarantees do not grow with the length of the trajectory. Our bounds depend on assumptions which are common in the literature; we provide an example framework where they would be verified to show that they are not as restrictive as they may seem.
\end{abstract}

\section{Introduction}
\label{sec:intro}
Many sectors rely on accurate time series forecasts: examples of these include the financial world, the medical domain and the energy sector. Each application has its own constraints, which has given rise to the development of multiple algorithms. Approaches range from those taken in traditional statistics, such as the autoregressive family of algorithms \citep{mulla_times_2024}, to deep learning such as recurrent neural networks and variants thereof \citep{rumelhart_learning_1986, hochreiter1997long, cho_properties_2014}, and, more recently, to tabular and time series foundation models such as those of \citet{Yuqietal_2023_PatchTST, ansari2024chronos, qu2025tabicl}. Since time series data can be very long and have many covariates, it is sometimes necessary to rely on feature engineering \citep{cerqueira_vest_2024}, or project them into a latent space \citep{wang_learning_2022, weisser2026cross}, for a model to output more precise predictions. Many approaches use an encoder-decoder architecture: for problems such as anomaly detection \citep{yokkampon_robust_2022} and data generation \citep{li_causal_2023}, in domains ranging from medical time series  to analysis of internet traffic \citep{li_anomaly_2021}. Variational Autoencoders \citep[VAEs,][]{kingma_vaes_2014, DBLP:journals/corr/Doersch16, MAL-056} provide many approaches for time series forecasting \citep{DBLP:journals/corr/abs-2111-08095, CAI2023111079}.

VAEs are widely used in many time series applications; however, few theoretical guarantees exist for this family of models. The theoretical works that exist include guarantees on the convergence of VAEs during training \citep{surendran_theoretical_2025}, and for robustly trained  VAEs \citep{camuto2021towards,barrett_certifiably_2022}. As for reconstruction guarantees, existing contributions are mainly PAC-Bayesian, in the works of \citet{cherief-abdellatif_pac-bayesian_2022,mbacke_statistical_2023}. However, neither of these papers covers time series, which is one of the contributions of this work.

Most existing works on VAEs, as well as on PAC-Bayesian bounds, focus on settings with independent and identically distributed data, often restricted to static data such as images. This work addresses time series, specifically exploring scenarios where the chosen variational distribution introduces dependencies. This calls for additional mathematical decisions: among these, practitioners must choose appropriate ways to factorise the time series distribution. Classical approaches include mean-field approximation \citep{Blei2017variational} and Markov chain factorisation \citep[Part II]{douc2014nonlinear}. This raises specific questions concerning the factorisation and stability of the chosen variational approximation. In particular, Markovian and backward factorisations provide a natural way to represent such dependencies while retaining a tractable recursive structure. The main contribution of this paper is to develop PAC-Bayesian guarantees for sequential VAEs, with an explicit dependence on the sample size and on the stability properties of the structured variational distribution. Under suitable mixing and regularity assumptions, the resulting bounds remain controlled as the trajectory length increases.

To complement the theoretical analysis, we discuss a lightweight discrete-latent instantiation of the considered framework. The purpose of this construction is not to propose a new forecasting architecture, but to exhibit a simple setting in which the assumptions underlying the theoretical results can be verified and the different terms of the bounds can be interpreted explicitly. In the context of time series forecasting, where datasets can admit a latent discrete representation \citep{fortuin2018deep, cohen_variational_2023}, Vector-Quantised VAEs \citep[VQ-VAEs,][]{van2017neural} provide a relevant modelling approach. On a finite latent state space, uniformly positive transition probabilities ensure that the strong-mixing assumptions hold, while simple distance-based parametrisations of the variational distributions makes the sensitivity of the variational kernels controllable. We introduce a structured discrete variational distribution built from the predictions of a pre-trained forecasting model and a finite reference set of trajectories. It is intended as a concrete illustration of  the PAC-Bayesian guarantees to a tractable sequential latent representation, not as an additional algorithmic contribution.

\textbf{This paper provides the following contributions.} 
\begin{itemize}
    \item We establish PAC-Bayesian reconstruction guarantees for sequential VAEs with structured variational distributions. Building on the conditional PAC-Bayesian approach of \citet{mbacke_statistical_2023}, we obtain a sharp concentration bound and extend it to Markovian variational distributions through an explicit stability analysis. Under suitable assumptions, the additional structural contribution remains uniformly controlled with respect to the trajectory length. The proposed analysis goes beyond uniformly bounded losses through a sub-gamma formulation, allowing the framework to cover a broader class of losses relevant to VAEs.
    \item We discuss a lightweight discrete-latent instantiation of the theoretical framework for which the Markovian structure and mixing assumptions can be made explicit. This construction is used as a proof of concept to illustrate how the theoretical conditions can guide the design of structured variational representations for time series.
\end{itemize}

\textbf{Outline.} The rest of this paper is organised as follows. \Cref{sec:background} briefly introduces existing work on VAEs, PAC-Bayesian generalisation bounds and excess risk bounds for state space models. \Cref{sec:theoretical} introduces the assumptions under which we establish PAC-Bayesian reconstruction guarantees for VAEs, and then discusses a discrete-latent instantiation illustrating the resulting theoretical framework. 
We draw conclusions and discuss avenues of future work in \Cref{sec:conclusion}.

\section{Background and Related Works}
\label{sec:background}
In the following, for all distributions (resp. probability density) $q$, we write $\E_q$ the expectation under the distribution (resp. probability density) $q$. For all sequences $(\mathbf{a}_u)_{u\in\mathbb{Z}}$ and all $s\leq t$, we write $\mathbf{a}_{s:t} = (a_s,\ldots, a_t)$. For all distributions $\decdistrib$ and $\encdistrib$ absolutely continuous with respect to a measure $\mu$ on a measurable space $(\mathsf{E},\mathcal{E})$, the \emph{Kullback-Leibler} (KL) divergence between $\decdistrib$ and $\encdistrib$ is defined as:
\begin{equation*}
    \kl(\decdistrib\|  \encdistrib) = \int_{} \decdistrib(u) \log \frac{\decdistrib(u)}{\encdistrib(u)} \mu(\rmd u) \eqsp.
\end{equation*}
\paragraph{Variational Autoencoders.} Consider a dataset $\{ \data^1_{\timerange}, \ldots, \data^\datasamples_{\timerange}\}$ with $\datasamples$ independent samples of time series of length $T\in\mathbb{N}$, with each time step $x_\inputindex^\samplesindex \in \mathbb{R}^\covardim$ containing $\covardim$ covariates.
The time steps are indexed with $\inputindex$ and the time series are indexed with $\samplesindex$. For all $1\leq i \leq \datasamples$,  $\data^i_{\timerange}$ has unknown distribution $\datagendis$ and depends on a latent variable $\mathbf{z}^i_{1:T}\in\mathbb{R}^{T\times d_\ell}$.

Consider a family of joint probability distributions $\{(\mathbf{z}_{\timerange},\mathbf{x}_{\timerange}) \mapsto  p_{\theta}(\mathbf{z}_{\timerange},\mathbf{x}_{\timerange})\}_{\theta\in\Theta}$ where $\Theta$ is a parameter space.
In this setting, we write, for all $\theta\in\Theta$, $\data_{\timerange}\in\mathbb{R}^{\datadim \times \covardim}$, $\latent_{\timerange}\in\mathbb{R}^{\datadim \times d_\ell}$,
\begin{equation*}
    p_{\theta}(\latent_{\timerange},\data_{\timerange}) = p_{\theta}(\latent_{\timerange})p_{\theta}(\data_{\timerange}|\latent_{\timerange})\eqsp.
\end{equation*}
The latent variable generative model defines a \emph{prior} $\latent_{1:T}\mapsto p_{\theta}(\latent_{1:T})$ over the latent variable $\latent_{\timerange}$ and a conditional density $\data_{1:T}\mapsto p_\theta(\data_{1:T}|\latent_{1:T})$, also known as the \emph{decoder}. The normalised loglikelihood is therefore given by
\begin{equation*}
    \ell_{\datasamples}(\theta) = \frac{1}{\datasamples}\sum_{\samplesindex=1}^\datasamples \log p_{\theta}(\data^{\samplesindex}_{\timerange}) 
    = \frac{1}{\datasamples}\sum_{\samplesindex=1}^\datasamples \log \int \decdistrib_{\paramdec}(\latent_{\timerange})\decdistrib_{\paramdec}(\data^i_{\timerange}|\latent_{\timerange})\rmd \latent_{\timerange}\eqsp,
\end{equation*}
and the conditional distribution is $\decdistrib_{\paramdec}(\latent_{\timerange}|\data_{\timerange})\propto \decdistrib_{\paramdec}(\latent_{\timerange})\decdistrib_{\paramdec}(\data_{\timerange}|\latent_{\timerange})$. In most cases, maximising the average marginal log-likelihood of the data is not possible.
As the conditional density $\decdistrib_{\paramdec}(\latent_{\timerange}|\data_{\timerange})$ is not available, consider a family of probability density functions $\{(\latent_{\timerange},\data_{\timerange}) \mapsto \encdistrib_{\varphi}(\latent_{\timerange}|\data_{\timerange})\}_{\varphi\in\Phi}$ that aim to approximate $\decdistrib_{\paramdec}(\latent_{\timerange}|\data_{\timerange})$. For all $\paramenc\in\Phi$, $q_{\paramenc}$ is referred to as an \emph{encoder}. The respective goals of the decoder and encoder are to learn a conditional likelihood $\decdistrib_{\paramdec}(\data_{\timerange} \mid \latent_{\timerange})$ which approximates the data-generating distribution $\datagendis$ on $\mathbb{R}^{\datadim \times \covardim}$ and to learn $\encdistrib_{\paramenc}(\latent_{\timerange} \mid \data_{\timerange})$ which approximates the intractable posterior $\decdistrib_{\paramdec}(\latent_{\timerange} \mid \data_{\timerange})$ on $\R^{\datadim \times d_\ell}$.

The encoder and decoder are trained jointly to maximise the Evidence Lower Bound (ELBO), a function of the data $\data_{\timerange}$ which depends on the number of time steps and a regularisation hyperparameter $\beta > 0$. The standard loss function \citep[negative ELBO,][]{higgins_-vae_2017} is given by  
\begin{equation}
\label{eq_elbo}
    \mathcal{L}_{\vae}(\paramenc, \paramdec) = \frac{1}{\datasamples} \sum_{\samplesindex = 1}^{\datasamples} \biggl[ \E_{\encdistrib_{\paramenc}(\cdot\mid\mathbf{x}^i_{1:T})} [-\log \decdistrib_{\paramdec}(\data_{\timerange}^\samplesindex \mid \cdot)] 
    + \beta \kl(\encdistrib_{\paramenc}(\cdot \vert \data_{\timerange}^\samplesindex)\| \decdistrib_{\paramdec} )\biggr] \eqsp.
\end{equation}
In general, the encoder and decoder are trained jointly to minimise an  empirical version of this loss:
\begin{equation}
\label{eq_empirical_elbo}
    \hat{\mathcal{L}}_{\vae}(\paramenc, \paramdec) = \frac{1}{\datasamples M} \sum_{\samplesindex = 1}^{\datasamples}\sum_{j=1}^M \biggl[ \ell_{\paramdec}(\mathbf{z}_{1:T}^{i,j}, \data_{\timerange}^{\samplesindex}) 
    + \beta \log \frac{\encdistrib_{\paramenc}(\mathbf{z}_{1:T}^{i,j} \vert \data_{\timerange}^\samplesindex)}{p_\theta(\mathbf{z}_{1:T}^{i,j})}\biggr] \eqsp,
\end{equation}
where for all $1\leq i \leq n$, $(\mathbf{z}_{1:T}^{i,j})_{1\leq j \leq M}$ are i.i.d. with distribution $\encdistrib_{\paramenc}(\cdot \vert \data_{\timerange}^\samplesindex)$ and  $\ell_{\paramdec} : \R^{\datadim \times d_\ell} \times \R^{\datadim \times \covardim} \rightarrow \R$ is a reconstruction loss which depends on the choice of model. 
Note the dimension of the latent variable is generally assumed to be much smaller than that of the data.

The loss can be chosen according to the model and assumptions on the data. In the case where under $\decdistrib_\paramdec\left(\data_{\timerange} \mid \latent_{\timerange}\right)$, the random variables $\data_{\timerange}$ are independent and Gaussian with mean function given for every $1\leq t \leq \datadim$ by $\decfunc(\latentpt)$ (a trainable decoding function which maps latent states to observations), we consider the classic squared loss $\ell_\paramdec(\latent_{\timerange}, \data_{\timerange}) = 
    \sum_{\inputindex = 1}^{\datadim}\ell_{\inputindex}(\datapt,\decfunc(\latentpt))/\datadim = \frac{1}{T}\sum_{t=1}^T(x_t - g_\theta(z_t))^2$.

\paragraph{PAC-Bayes Bounds.}
\label{subsec_pacbayes}
In machine learning, generalisation is a model's capacity to perform well on unseen data. Various statistical methods can be used to provide  generalisation guarantees, which are presented as a bound on a loss function. This guarantees that the performance will be suitable on a new dataset. There are multiple ways to derive such bounds, which exist, amongst others, in information theory \citep{hellstrom_generalization_2025}. The Probably Approximately Correct (PAC)-Bayesian framework provides many such bounds, see \citet{alquier2024user} and \citet{guedj_primer_2019} for an introduction.

PAC-Bayesian theory focuses on the generalisation capacity of a randomly selected posterior  from a set of possible posteriors by bounding the gap between the theoretical and empirical risks.
The theoretical risk of a VAE, for a given loss function $\ell$, of the posterior distribution $\encdistrib_\paramenc(\cdot\vert \data_{\timerange})$, with $\paramenc \in \Phi$, is 
\begin{equation*}
    \risk(\paramenc) = \E_{\mathcal{D}} \biggl[ \E_{\encdistrib_{\paramenc}(\cdot \vert \data_{\timerange})} \left[ \ell(\cdot, \data_{\timerange}) \right] \biggr]\eqsp,
\end{equation*}
and the empirical risk is
\begin{equation*}
    \widehat\risk_n(\paramenc) = \frac{1}{\datasamples} \sum_{\samplesindex=1}^{\datasamples}\E_{\encdistrib_\paramenc(\cdot \vert \data_{\timerange}^\samplesindex)} \biggl[  \ell(\cdot, \data_{\timerange}^\samplesindex) \biggr]\eqsp.
\end{equation*}
Given the empirical risk on an observed sample, PAC-Bayesian theory provides probabilistic bounds on risk.
Non-vacuous PAC-Bayes generalisation bounds exist for a variety of deep learning architectures. Initially introduced by \citet{dziugaite_computing_2017}, a wealth of non-vacuous PAC-Bayesian deep learning bounds now exist. This includes bounds that are minimised when training neural networks \citep{letarte_dichotomize_2019,biggs2020differentiable,perez-ortiz2021tighter,biggs2021margins,biggs2022shallow}. Bounds exist for multiple different generative architectures, such as diffusion models \citep{mbacke_note_2024} and adversarial generative models \citep{mbacke_pac-bayesian_2023}.

PAC-Bayesian guarantees exist for VAE reconstruction loss \citep{cherief-abdellatif_pac-bayesian_2022} as well as for generation and regeneration \citep{mbacke_statistical_2023}. 
\citet{cherief-abdellatif_pac-bayesian_2022} provides a learning objective on the reconstruction gap and 
\citet{mbacke_statistical_2023} introduce a PAC-Bayesian bound for a conditional posterior distribution, which is then used alongside hypotheses on the data space to bound reconstruction loss and provide generation and regeneration guarantees for a VAE. 
The bound laid out within \Cref{lemma_conditional_hoeffding} reinterprets the objective from \citet{mbacke_statistical_2023} within the framework of the present paper on latent variable models that encode the dependencies and structure of an observed time series.

\paragraph{Excess-risk bounds for dependent data.} Variational excess-risk bounds have also been proposed recently for general state space models. Following \cite{campbell2021online}, \cite{JMLR:v25:22-1392}
provided upper bounds on the latent state estimation error using a backward factorization of the variational distribution. Using the same approach as in \cite{tang2021empirical},  \cite{gassiat2024variational} established an oracle inequality for the risk, explicit in particular in the number of samples and
in the length of the observation sequences, and \cite{fassina2026generalizing} provided a theoretically grounded KL-driven initialization of score-based diffusion models. Within the PAC-Bayesian framework, \citet{karagulyan2026empirical} provide bounds for Markov Chains that depend on a pseudo-spectral gap, and its estimator for an empirical version of the bound. However, obtaining statistical guarantees for VAEs in state space models using PAC-Bayesian theory remains largely an open problem. In this paper, we establish generalisation bounds to obtain theoretically grounded reconstruction and sampling guarantees in this setting. This framework with dependent data is also particularly relevant, not only for practical applications, but also because it ensures identifiability under suitable conditions \citep{khemakhem2020variational}, and for sequential data \citep{gassiat2020identifiability}.

\section{PAC-Bayesian Bounds on Reconstructing Time Series}
\label{sec:theoretical}
In order to derive generalisation bounds for state space models, the expected risk is initially bounded in \Cref{prop:conditional_hoeffding}. This key result only requires the loss function to be bounded, a classic assumption in PAC-Bayesian theory. 
The resulting bound provides a tighter guarantee building on the work of \citet{mbacke_statistical_2023}. It is then combined with a smoothness assumption similar to that of \citet{mbacke_statistical_2023}, which controls the diameter of the observation space: this assumption is then used to establish the main result, \Cref{prop:main:markov}, under additional structural assumptions on the distribution $\encdistrib_\paramenc$. These assumptions yield a richer bound that includes the mean-field case. Under a sub-gamma assumption for the reconstruction losses, \Cref{prop:conditional} avoids bounded losses and provides a bound which goes beyond the boundedness assumption or the control on the observation space.

\begin{proposition}
\label{prop:conditional_hoeffding}
Assume that the loss function $\ell_\theta:\R^{T\times d_\ell}\times\R^{T\times\covardim}\rightarrow\R^+$ is such that there exists $b>0$ satisfying, for all $\theta\in\Theta$, $\latent_{\timerange}\in\R^{T\times d_\ell}$ and $\data_{\timerange}\in\R^{T\times\covardim}$, $\ell_\theta(\latent_{\timerange},\data_{\timerange})\leq b$. 
Let $\theta\in\Theta$ be fixed independently of the observations and let $\decdistrib$ be a prior distribution on $\R^{T\times d_\ell}$ which does not depend on the observed dataset. Then, for all $\lambda_\datasamples>0$ and $\eta_\datasamples\in(0,1)$, with probability at least $1-\eta_\datasamples$ with respect to $\dataset=\left\{\data_{\timerange}^\samplesindex\right\}_{\samplesindex=1}^{\datasamples}\sim\datagendis^{\otimes\datasamples}$, the following inequality holds simultaneously for all conditional variational distributions $\encdistrib(\cdot\mid\data_{\timerange}^\samplesindex)$ satisfying $\encdistrib(\cdot\mid\data_{\timerange}^\samplesindex)\ll\decdistrib$:
\begin{multline}
\label{inequ_conditional_bound}
\frac{1}{\datasamples}\sum_{\samplesindex=1}^{\datasamples}\E_{\encdistrib(\cdot\mid\data_{\timerange}^\samplesindex)}\left[\E_{\data_{\timerange}\sim\datagendis}\left[\ell_\theta(\latent_{\timerange},\data_{\timerange})\right]\right]-\frac{1}{\datasamples}\sum_{\samplesindex=1}^{\datasamples}\E_{\encdistrib(\cdot\mid\data_{\timerange}^\samplesindex)}\left[\ell_\theta(\latent_{\timerange},\data_{\timerange}^\samplesindex)\right]\\
\leq\frac{1}{\lambda_\datasamples}\sum_{\samplesindex=1}^{\datasamples}\kl\left(\encdistrib(\cdot\mid\data_{\timerange}^\samplesindex)||\decdistrib\right)+\frac{\lambda_\datasamples b^2}{8\datasamples}+\frac{\log(1/\eta_\datasamples)}{\lambda_\datasamples}\eqsp.
\end{multline}
\end{proposition}

\begin{proof}
    The proof is postponed to \Cref{subsec:proof:main}.
\end{proof}

In most settings, $\ell_\paramdec(\latent_{\timerange}, \data_{\timerange}) = \sum_{t=1}^{\datadim}\ell_{\theta,t}(\latentpt,\datapt)/\datadim$, so that bounding the loss amounts to considering bounded functions $\ell_{\theta,t}$, for any $1\leq t\leq \datadim$. The boundedness assumption is classical in PAC-Bayes theory \citep[][among others]{dziugaite_data-dependent_2018, rivasplata_pac-bayes_2020} and allows here for a more concise treatment of our theoretical results. Note however that a growing body of work, such as \citet{alquier2018simpler,haddouche2020pacbayes,haddouche2022supermartingales}, aims at relaxing the boundedness assumption, with techniques relying on assumptions such as sub-Gaussian or sub-gamma losses. The latter assumption is applied to the bound in \Cref{prop:conditional}.

\begin{corollary}
\label{cor:conditional:hoeffding:optimized}
Assume that the assumptions of Proposition~\ref{prop:conditional_hoeffding} hold. Assume also that there exists $K_\datasamples>0$ such that $\sum_{\samplesindex=1}^{\datasamples}\kl\left(\encdistrib(\cdot\mid\data_{\timerange}^\samplesindex)||\decdistrib\right)\leq K_\datasamples$. 
Let $\eta_\datasamples\in(0,1)$ and choose
$$
\lambda_\datasamples=\sqrt{\frac{8\datasamples\{ K_\datasamples+\log(1/\eta_\datasamples)\}}{b^2}}\eqsp.
$$
Then, with probability at least $1-\eta_\datasamples$ with respect to $\dataset\sim\datagendis^{\otimes\datasamples}$,
\begin{equation}
\label{eq:conditional:hoeffding:optimized}
\frac{1}{\datasamples}\sum_{\samplesindex=1}^{\datasamples}\E_{\encdistrib(\cdot\mid\data_{\timerange}^\samplesindex)}\left[\E_{\data_{\timerange}\sim\datagendis}\left[\ell_\theta(\latent_{\timerange},\data_{\timerange})\right]\right]-\frac{1}{\datasamples}\sum_{\samplesindex=1}^{\datasamples}\E_{\encdistrib(\cdot\mid\data_{\timerange}^\samplesindex)}\left[\ell_\theta(\latent_{\timerange},\data_{\timerange}^\samplesindex)\right]\
\leq b\sqrt{\frac{ K_\datasamples+\log(1/\eta_\datasamples)}{2\datasamples}}\eqsp.
\end{equation}
\end{corollary}

\begin{proof}
By Proposition~\ref{prop:conditional_hoeffding}, with probability at least $1-\eta_\datasamples$,
$$
\frac{1}{\datasamples}\sum_{\samplesindex=1}^{\datasamples}\E_{\encdistrib(\cdot\mid\data_{\timerange}^\samplesindex)}\left[\E_{\data_{\timerange}\sim\datagendis}\left[\ell_\theta(\latent_{\timerange},\data_{\timerange})\right]\right]-\frac{1}{\datasamples}\sum_{\samplesindex=1}^{\datasamples}\E_{\encdistrib(\cdot\mid\data_{\timerange}^\samplesindex)}\left[\ell_\theta(\latent_{\timerange},\data_{\timerange}^\samplesindex)\right]\
\leq\frac{ K_\datasamples+\log(1/\eta_\datasamples)}{\lambda_\datasamples}+\frac{\lambda_\datasamples b^2}{8\datasamples}\eqsp.
$$
The right-hand side is minimised over $\lambda_\datasamples>0$ at
$$
\lambda_\datasamples=\sqrt{\frac{8\datasamples\{ K_\datasamples+\log(1/\eta_\datasamples)\}}{b^2}}\eqsp,
$$
which yields
$$
\frac{ K_\datasamples+\log(1/\eta_\datasamples)}{\lambda_\datasamples}=\frac{\lambda_\datasamples b^2}{8\datasamples}=b\sqrt{\frac{ K_\datasamples+\log(1/\eta_\datasamples)}{8\datasamples}}\eqsp,
$$
and concludes the proof.
\end{proof}

\emph{Choice of the confidence level and convergence rates.}

Assume that $ K_\datasamples=O(\datasamples^\kappa)$ for some $\kappa\in[0,1)$. If the confidence level is fixed, \emph{i.e.}, $\eta_\datasamples=\eta\in(0,1)$, then
$\lambda_\datasamples=O(\datasamples^{(1+\kappa)/2})$
and the upper bound in \eqref{eq:conditional:hoeffding:optimized} is of order
$O(\datasamples^{(\kappa-1)/2})$. 
In particular, if $ K_\datasamples=O(1)$, then
$\lambda_\datasamples=O(\sqrt{\datasamples})$ and the upper bound in 
\eqref{eq:conditional:hoeffding:optimized} is $O(\datasamples^{-1/2})$. More generally, if $\eta_\datasamples=\datasamples^{-\alpha}$ for some $\alpha>0$, then
$$
\lambda_\datasamples=\sqrt{\frac{8\datasamples\{ K_\datasamples+\alpha\log\datasamples\}}{b^2}}
$$
and the upper bound is  $O(((K_\datasamples+\log\datasamples)/\datasamples)^{1/2})$.
Consequently, if $ K_\datasamples=O(1)$, the resulting rate is
$O((\log\datasamples /\datasamples)^{1/2})$. 
If $ K_\datasamples=O(\datasamples^\kappa)$ for some $\kappa\in(0,1)$, the polynomial contribution dominates the logarithmic confidence term and the rate remains $O(\datasamples^{(\kappa-1)/2})$.

\paragraph{Upper bound for structured variational distributions. }
In this context where the observations are $\mathbf{x}_{1:T}$, consider a parametric family of variational distributions denoted as $\vd{1:T,x}$, \emph{i.e.}, the unknown variational parameter is $\parvar \in\Phi$ and the input data is written $x$ for better readability, with observations at time steps $1:T$.  Recently, \cite{campbell2021online,JMLR:v25:22-1392} proposed a backward factorisation of the variational family by introducing, for all $\mathbf{x}_{1:T}\in\mathbb{R}^{\datadim \times \covardim}$ with latent representation $\mathbf{z}_{1:T}\in\mathbb{R}^{\datadim \times \covardim_\ell}$,
\begin{equation}
\label{eq:backward:factorization}
 \vd{1:T,x}(\mathbf{z}_{1:T})=  \vd{T,x}(z_T)\prod_{k=2}^{T}\vd{k-1\vert k,x}(z_{k},z_{k - 1})\eqsp,
\end{equation}
where $\vd{T,x}$ (resp. $\vd{k - 1\vert k,x}(z_{k},\cdot)$) are user-chosen p.d.f. which may depend on the observations. Such a variational family allows the introduction of time-dependent latent variables (such as Markov chains) while defining the variational distribution recursively, with shared parameters for all kernels. 
This new design based on backward factorization is crucial for online parameter learning, see for instance \cite{chagneux2026efficient}.  

The theoretical analysis of the proposed PAC-Bayes generalisation bounds is performed under the strong-mixing assumption H\ref{hyp:strongmixing}. It is a classical assumption in the state-space models literature, in particular to establish long-term stochastic stability of smoothing algorithms, see \emph{e.g.} \citet[Chapter 4]{moral2004feynman}, \citet[Section 4.3]{cappe2005inference}, \citet{gloaguen2022pseudo} and references therein, and is satisfied when the latent state space is compact under suitable positivity/regularity assumptions on the transition densities. Such assumptions are also common in discrete state space models, see for instance \cite{de2017consistent} for the control of filtering and smoothing distributions in discrete hidden Markov models.

\begin{hypH} 
    \label{hyp:strongmixing}
    For all $\mathbf{x}_{1:T}\in\mathbb{R}^{\datadim \times \covardim}$, for $2\leq k \leq T$, there exist $0<\underline{\sigma}^\parvar_{k}(\mathbf{x}_{1:T})<\overline{\sigma}^\parvar_{k}(\mathbf{x}_{1:T})<\infty$ such that for all $\parvar \in\Phi$, $z_{k-1}$, $z_k$,
    $$
    \underline{\sigma}^\parvar_{k}(\mathbf{x}_{1:T})\leq \vd{k-1\vert k,x}(z_{k},z_{k - 1}) \leq \overline{\sigma}^\parvar_{k}(\mathbf{x}_{1:T})\eqsp.
    $$
    Let $\rho^\parvar_{k}(\mathbf{x}_{1:T}) = 1 - \underline{\sigma}^\parvar_{k}(\mathbf{x}_{1:T})/\overline{\sigma}^\parvar_{k}(\mathbf{x}_{1:T})$.
\end{hypH}
The main result provided in Proposition~\ref{prop:main:markov} establishes a generalisation bound in a case where the variational distribution has a Markovian structure as in \eqref{eq:backward:factorization}. This result requires to control the bias of the variational approximation of expectations of additive state functionals between different sequences of observations. This control is provided in Proposition~\ref{prop:MC:df}, which establishes that this error grows at most linearly in the number of observations under H\ref{hyp:strongmixing}.  

\begin{hypH} 
There exist $C^\parvar_{T}>0$, $C^\parvar_{k}$, $1\leq k \leq T-1$, and positive functions $d_k$, $1\leq k \leq T$, such that for all $\mathbf{x}_{1:T},\mathbf{x'}_{1:T}\in\mathbb{R}^{\datadim \times \covardim}$, all bounded measurable functions $\psi$, and all $\parvar \in\Phi$,
    \label{hyp:lip}
\begin{equation*}
         \left|\vd{T,x}[\psi]-\vd{T,x'}[\psi]\right|\\
         \leq C^\parvar_{T}\mathrm{d}_T(\mathbf{x}_{1:T},\mathbf{x}_{1:T}')\|\psi\|_\infty\eqsp.
    \end{equation*}
    and for all $2\leq k \leq T$,
    \begin{equation*}
         \left|\vd{k-1\vert k,x}(z_{k},\psi)-\vd{k-1\vert k,x'}(z_{k},\psi)\right| 
         \leq C^\parvar_{k-1}(z_{k})\mathrm{d}_{k-1}(\mathbf{x}_{1:T},\mathbf{x}_{1:T}')\|\psi\|_\infty\eqsp.
    \end{equation*}
 \end{hypH}   
In the independent setting, Assumption 2 of \cite{mbacke_statistical_2023} states that the encoder and decoder networks have finite Lipschitz norms. This is a common assumption for standard neural network-based VAEs, see for instance \cite{surendran_theoretical_2025} for theoretical guarantees of gradient-based training of VAEs. Assumption H\ref{hyp:lip} provides a similar setting for dependent data where the distances $d_k$, $1\leq k \leq T$ depend on the network architectures.

\paragraph{Application to  deep Gaussian mean field and recurrent variational family.}
Assumption H\ref{hyp:lip} is an extension of Assumption 3 in \cite{mbacke_statistical_2023} in the inhomogeneous Markovian case. Mean-field variational approximation and variational approximation based on recurrent neural networks are very common settings. We consider the following frameworks which show how to satisfy Assumption H\ref{hyp:lip}; they do not, however, satisfy H\ref{hyp:strongmixing} on an unbounded Gaussian latent space.
\begin{itemize}
    \item {\bf Mean-field. }In this setting, for all $2\leq k \leq T$, $z_k\mapsto \vd{k-1\vert k,x}(z_k,\cdot)$, does not depend on $z_k$ and is written $\vd{k-1,x}(\cdot)$. Assume that $\vd{k-1,x}$ is a Gaussian density with mean $\mu_{k,\parvar}(x_k)$ and diagonal covariance matrix with diagonal $\sigma^2_\parvar I$. Assume that for all $2\leq k \leq T$, $\mu_{k,\parvar}$  are bounded functions such that $\|\mu_\parvar(x_k)-\mu_\parvar(x'_k)\|\leq K_\parvar \|x_k-x_k'\|$. Then, H\ref{hyp:lip} holds with $d_{k-1}(\mathbf{x}_{1:T},\mathbf{x'}_{1:T}) = \|x_k-x'_k\|$.
    \item {\bf Recurrent. } In this setting, for all $2\leq k \leq T$, $z_k\mapsto \vd{k-1\vert k,x}(z_k,\cdot)$ is a Gaussian density with mean $\tilde{\mu}_{k,\parvar}(x_{k},z_k)$ and diagonal covariance matrix with diagonal $\sigma^2_\parvar I$. Assume that for all $2\leq k \leq T$,  $\tilde{\mu}_{k,\parvar}$ are bounded functions such that $\|\tilde\mu_{k,\parvar}(z_k,x_{k})-\tilde\mu_{k,\parvar}(z_k,x'_{k})\|\leq K_\parvar(z_k) \|x_{k}-x_{k}'\|$.
Then, H\ref{hyp:lip} holds with $d_{k-1}(\mathbf{x}_{1:T},\mathbf{x'}_{1:T}) = \|x_k-x'_k\|$.
\end{itemize} 

Define the space of additive state functionals:
\begin{equation}
    \label{eq:def:erond}
    \mathcal{E} = \left\{f:\latentspace \to \R_+\;\;  f: \mathbf{z}_{1:T} \mapsto \sum_{k=1}^T f_k(z_k)\right\}\eqsp.
\end{equation}
In Proposition~\ref{prop:main:markov}, we assume that the loss function is such that for all $\data_{1:T}$, $\latent_{1:T} \mapsto \ell_\theta (\latent_{1:T},\data_{1:T})$ is in $\mathcal{E}$, i.e. $\ell (\latent_{1:T},\data_{1:T})= \sum_{k=1}^T \ell_{\theta,k}(z_k,\data_{1:T})/T$.

\begin{proposition}
\label{prop:main:markov}
Assume that the loss function $\ell_\theta:\R^{T\times d_\ell}\times\R^{T\times\covardim}\rightarrow\R^+$ is such that there exists $b>0$ satisfying, for all $\theta\in\Theta$, $\latent_{\timerange}\in\R^{T\times d_\ell}$ and $\data_{\timerange}\in\R^{T\times\covardim}$,
$\ell_\theta(\latent_{\timerange},\data_{\timerange})\leq b$.
Assume also that H\ref{hyp:strongmixing}-\ref{hyp:lip} hold and that the variational distribution $\encdistrib_\paramenc$ is parametrised as in \eqref{eq:backward:factorization}. Assume that, for all $\data_{1:T}$, $\latent_{1:T}\mapsto\ell_\theta(\latent_{1:T},\data_{1:T})$ belongs to $\mathcal{E}$, with
$$
\ell_\theta(\latent_{1:T},\data_{1:T})=\frac{1}{T}\sum_{k=1}^T\ell_{\theta,k}(z_k,\data_{1:T})\eqsp.
$$
Assume that there exist $M>0$, $C>0$, $\underline{\sigma}^\parvar>0$ and $\overline{\sigma}^\parvar<\infty$ such that, for all $1\leq k\leq T$,
$$
\|\ell_{\theta,k}\|_\infty\leq M,\qquad \|C^\parvar_k\|_\infty\leq C,\qquad \overline{\sigma}^\parvar_k(\data_{1:T})\leq\overline{\sigma}^\parvar,\qquad \underline{\sigma}^\parvar_k(\data_{1:T})\geq\underline{\sigma}^\parvar\eqsp.
$$
Let $\theta\in\Theta$ be fixed independently of the observations and let $\decdistrib$ be a prior distribution which does not depend on the observed dataset and assume that
$
\encdistrib_\paramenc(\cdot\mid\data_{\timerange}^\samplesindex)\ll\decdistrib$, for $1\leq\samplesindex\leq\datasamples$. 
Then, for all $\lambda_\datasamples>0$ and $\eta_\datasamples\in(0,1)$, with probability at least $1-\eta_\datasamples$ with respect to
$\dataset=\left\{\data_{\timerange}^\samplesindex\right\}_{\samplesindex=1}^{\datasamples}\sim\datagendis^{\otimes\datasamples}$, 
\begin{align}
\E_{\data_{\timerange}\sim\datagendis}&\left[\E_{\encdistrib_\paramenc(\cdot\mid\data_{\timerange})}\left[\ell_\theta(\latent_{\timerange},\data_{\timerange})\right]\right]
-\frac{1}{\datasamples}\sum_{\samplesindex=1}^{\datasamples}\E_{\encdistrib_\paramenc(\cdot\mid\data_{\timerange}^\samplesindex)}\left[\ell_\theta(\latent_{\timerange},\data_{\timerange}^\samplesindex)\right]\notag\\
&\leq\frac{1}{\lambda_\datasamples}\sum_{\samplesindex=1}^{\datasamples}\kl\left(\encdistrib_\paramenc(\cdot\mid\data_{\timerange}^\samplesindex)||\decdistrib\right)+\frac{CM}{\datasamples T(1-\rho^\parvar)}\sum_{\samplesindex=1}^{\datasamples}\sum_{s=1}^T\E_{\data_{\timerange}\sim\datagendis}\left[\mathrm{d}_s(\data_{1:T},\data_{1:T}^\samplesindex)\right]+\frac{\lambda_\datasamples b^2}{8\datasamples}+\frac{\log(1/\eta_\datasamples)}{\lambda_\datasamples}\eqsp.
\label{eq:main}
\end{align}
\end{proposition}

\begin{proof}
    The proof is postponed to \Cref{subsec:pf_main}.
\end{proof}

\Cref{prop:main:markov} bounds the generalisation gap with three terms. The first is a sum of KL divergences between the distributions conditioned on the chosen samples, and the chosen prior distributions. 
In the same setting as \citet[Theorem~4.3]{mbacke_statistical_2023}, if the observation space is compact with diameter $\Delta$, and if for all $1\leq s\leq T$, $\mathrm{d}_{s} (\mathbf{x}_{1:T},\mathbf{x}^i_{1:T})$ depends only on $\mathbf{x}_{s}$ and $\mathbf{x}^i_{s}$ and is upper bounded by $\Delta$, the second term is upper bounded by $CM\Delta /(1-\rho^\parvar)$. 
Following the same discussion as for \Cref{prop:conditional_hoeffding},
assume that $\overline K_\datasamples=O(\datasamples^\kappa)$ for some $\kappa\in[0,1)$. If the confidence level is fixed, \emph{i.e.}, $\eta_\datasamples=\eta\in(0,1)$, then
$\lambda_\datasamples=O(\datasamples^{(1+\kappa)/2})$
and the upper bound in \eqref{eq:conditional:hoeffding:optimized} is of order
$O(\datasamples^{(\kappa-1)/2})$. Then, there exist a constant $\tilde C>0$ such that   
\begin{equation*}
        \E_{\datagendis}\left[\E_{\encdistrib_\paramenc\left(\cdot \mid \data_{\timerange}\right)}\left[\ell\left(\latent_{\timerange}, \data_{\timerange}\right)\right] \right] - \frac{1}{\datasamples}\sum_{\samplesindex=1}^\datasamples \E_{\encdistrib_\paramenc\left(\cdot \mid\data_{\timerange}^\samplesindex\right)} \left[ \ell(\latent_{\timerange}, \data_{\timerange}^\samplesindex)\right] \leq    \frac{CM\Delta}{1-\rho^\parvar} + \tilde C \datasamples^{(\kappa-1)/2} \eqsp.
    \end{equation*}
In the setting where $\ell_\paramdec(\latent_{\timerange}, \data_{\timerange}) = \sum_{k=1}^{\datadim}\ell_{\theta,k}(z_k,x_k)/\datadim$ with $\|\ell_{\theta,k}\|_\infty \leq M$, the second term of the upper bound is uniformly bounded in time and scales similarly as in \citet[Theorem~4.3]{mbacke_statistical_2023}. However, contrary to them, the $1/T$ normalisation of the reconstruction loss compensates for the linear growth. Therefore, under these assumptions, leveraging structured variational distributions enables the derivation of bounds that remain stable as the trajectory length increases.

Propositions~\ref{prop:conditional_hoeffding} and~\ref{prop:main:markov} extend the conditional PAC-Bayesian framework introduced by \citet{mbacke_statistical_2023} to structured variational distributions for sequential latent variable models. The bounded-space reconstruction guarantee of \citet[Theorem~4.3]{mbacke_statistical_2023} provides an upper bound of the  form
$$
\Delta_n = \frac{K_\datasamples+\log(1/\eta_\datasamples)}{\lambda_\datasamples}
+K_\paramenc K_\theta\Delta
+\frac{\lambda_\datasamples\Delta^2}{8\datasamples}\eqsp,
$$
where $K_\paramenc$ and $K_\theta$ are global Lipschitz constants of the encoder and decoder and $\Delta$ is the diameter of the observation space. Our PAC-Bayesian concentration contribution has a similar dependence on the sample size, but Proposition~\ref{prop:conditional_hoeffding} requires neither Lipschitz continuity of the encoder nor a bounded observation space beyond the boundedness of the loss itself. 

The main improvement provided by Proposition~\ref{prop:main:markov} is that in place of the global term $K_\paramenc K_\theta\Delta$, the proposed bound involves the time-local stability contribution $CM
\sum_{\samplesindex=1}^{\datasamples}\sum_{s=1}^T
\E_{\data_{1:T}\sim\datagendis}
\left[
\mathrm d_s(\data_{1:T},\data_{1:T}^\samplesindex)
\right]/ (\datasamples T(1-\rho^\parvar))$. 
This term makes explicit the local sensitivity of each backward variational kernel to the observations and the contraction properties of the resulting latent Markov chain. In particular, if $\mathrm d_s(\data_{1:T},\data'_{1:T})\leq\Delta$ uniformly, then this contribution is bounded by $CM\Delta/(1-\rho^\parvar)$, 
uniformly in the trajectory length $T$. This contrasts with a direct application of a vector-valued bound to complete trajectories, for which a global diameter or Lipschitz constant may deteriorate with the trajectory length. The proposed result therefore preserves the standard PAC-Bayesian statistical rate while providing a refined and time-uniform description of the cost induced by structured variational inference.

\Cref{prop:main:markov} provides an upper bound for structured variational distributions for dependent data. This theoretical guarantee completes the excess-risk bounds of \cite{gassiat2024variational} in the PAC-Bayes setting. This work focuses on the case where the structure of the variational distribution matches the true generative model of a state-space model. This opens various perspectives to provide guarantees for specific  parametrisations of the variational distributions as designing efficient approximations remains an active area of research, see \cite{johnson2016composing,lin2018variational,campbell2021online}.

\paragraph{PAC-Bayesian bounds for sub-gamma reconstruction losses.}
The boundedness assumption of the loss is restrictive for standard VAE losses. In particular, squared reconstruction losses and negative log-likelihood losses are generally unbounded when state spaces are not compact or bounded. 

\begin{hypH}
\label{hyp:subgamma}
There exist constants $v>0$ and $c> 0$ such that, for all $\theta\in\Theta$, $\latent_{\timerange}\in\R^{T\times d_\ell}$ and $\alpha\in(0,1/c)$,
$$
\log\E_{\data_{\timerange}\sim\datagendis}\left[\exp\left(\alpha\left\{\E_{\data_{\timerange}\sim\datagendis}\left[\ell_\theta(\latent_{\timerange},\data_{\timerange})\right]-\ell_\theta(\latent_{\timerange},\data_{\timerange})\right\}\right)\right]\leq\frac{\alpha^2v}{2(1-c\alpha)}\eqsp.
$$
\end{hypH}

Proposition~\ref{prop:conditional} complements the bounded-space and manifold-based guarantees of \citet{mbacke_statistical_2023} by allowing unbounded reconstruction losses under a uniform sub-gamma condition. This provides a model-agnostic route for handling squared reconstruction losses and negative log-likelihood losses on non-compact spaces, without requiring the data-generating distribution to be explicitly represented as the pushforward of a low-dimensional reference distribution.

\begin{proposition}
\label{prop:conditional}
Assume that H\ref{hyp:subgamma} holds. Let $\theta\in\Theta$ be fixed independently of the observations and let $\decdistrib$ be a prior distribution on $\R^{T\times d_\ell}$ which does not depend on the observed dataset. Then, for all $\eta_\datasamples\in(0,1)$ and all
$0<\lambda_\datasamples<\datasamples/c$, with probability at least $1-\eta_\datasamples$ with respect to $\dataset\sim\datagendis^{\otimes\datasamples}$, the following inequality holds simultaneously for all conditional variational distributions $\encdistrib(\cdot\mid\data_{\timerange}^\samplesindex)$ satisfying $\encdistrib(\cdot\mid\data_{\timerange}^\samplesindex)\ll\decdistrib$:
\begin{multline}
\label{eq:conditional}
\frac{1}{\datasamples}\sum_{\samplesindex=1}^{\datasamples}\E_{\encdistrib(\cdot\mid\data_{\timerange}^\samplesindex)}\left[\E_{\data_{\timerange}\sim\datagendis}\left[\ell_\theta(\latent_{\timerange},\data_{\timerange})\right]\right]-\frac{1}{\datasamples}\sum_{\samplesindex=1}^{\datasamples}\E_{\encdistrib(\cdot\mid\data_{\timerange}^\samplesindex)}\left[\ell_\theta(\latent_{\timerange},\data_{\timerange}^\samplesindex)\right]\\
\leq\frac{1}{\lambda_\datasamples}\sum_{\samplesindex=1}^{\datasamples}\kl\left(\encdistrib(\cdot\mid\data_{\timerange}^\samplesindex)||\decdistrib\right)+\frac{\lambda_\datasamples v}{2\datasamples(1-c\lambda_\datasamples/\datasamples)}+\frac{\log(1/\eta_\datasamples)}{\lambda_\datasamples}\eqsp.
\end{multline}
\end{proposition}

\begin{proof}
    The proof is postponed to \Cref{subsec:pf_subgamma}.
\end{proof}

\paragraph{Application to a structured, lightweight, discrete latent encoder distribution. }
\label{sec:framework}
The previous results identify two main structural ingredients to obtain reconstruction guarantees that remain controlled for long sequences: a stable Markovian variational representation and a sufficiently regular dependence of the corresponding transition kernels on the observations. To illustrate the scope of the theory, we now describe a simple instantiation designed to allow explicit verification of the structural assumptions underlying our bounds. 
For a time series $\mathbf{x}_{1:T}\in\R^{T\times d}$, let $\mathbf{z}_{1:T}\in\R^{T\times d_\ell}$ be a set of latent variables that encode the hidden structure and temporal dependencies of the observed sequence. The model depends on an unknown parameter $\theta\in\Theta$, such that for all $\theta$, the joint distribution of the latent data and the observations writes $\decdistrib_\paramdec(\series_{\timerange}, \latent_{\timerange}) = \decdistrib_\paramdec(\latent_{\timerange})\decdistrib_\paramdec(\series_{\timerange} \mid \latent_{\timerange})$.

Let $\dataset = \{ \mathbf{x}_{\timerange}^b\}_{b=1}^B$ be a reference dataset of $\datasamples$ independent time series of length $\datadim$. Following the few-shot learning approach, this reference dataset influences generalisation, and guides the proposed model to adapt to a new task. Discrete latent variables represent the underlying state space. Discrete representations are often more interpretable and naturally suited to modelling time series where the dynamics are governed by a finite set of regimes or modes. 
This approach is supported by models such as the Vector Quantised-VAE  \citep[VQ-VAE,][]{van2017neural,cohen2022diffusion}. In our setting, we assume that for all $t\in\{1,\ldots,T\}$, $z_t\in\{1,\ldots,B\}$. We thus use a lightweight, structured, VAE-like model which illustrates the limited growth in empirical risk relative to the length of time series $T$ or to the number of samples $n$, as predicted by the bound in \Cref{prop:main:markov}. 

Although the strong-mixing assumption H\ref{hyp:strongmixing} is challenging in non-compact state-spaces, it can be satisfied easily for discrete spaces, which makes this framework appealing. The finite latent space is particularly convenient from the perspective of Proposition~\ref{prop:main:markov}: if the transition probabilities are uniformly positive, the strong-mixing condition follows directly, while bounded and Lipschitz distance-based transition scores provide a simple mechanism for verifying H\ref{hyp:lip}. The methodological approach adopted here is primarily intended to highlight that the error in $T$ grows only linearly when using a structured variational distribution. A specific variational family designed to exhibit this behaviour is therefore introduced, with the objective of providing a proof of concept illustration of the theorem. The resulting structured variational distribution provides a simple mechanism for combining information from the pre-trained predictor and the reference trajectories while retaining the assumptions required by the theoretical analysis.

\emph{Sampling.} 
For each time series, decompose the joint distribution of the latent variable and the observation at a given time step as follows:
\begin{equation*}
\decdistrib_\paramdec\left(\seriestep_t,\latentpt\mid \latent_{1:t-1}, \series_{1:t-1}\right)
= \decdistrib_\paramdec\left(\latentpt\mid \latent_{1:t-1}, \series_{1:t-1}\right)\decdistrib_\paramdec\left(\seriestep_t\mid \latent_{1:t}, \series_{1:t-1}\right) \eqsp.
\end{equation*}
The distribution $\decdistrib_\paramdec\left(\datapt\mid \latent_{1:t}, \series_{1:t-1}\right)$ is referred to as the \emph{decoder} and is given by
\begin{equation*}
\decdistrib_\paramdec\left(\seriestep_t \mid \latent_{1:t}, \series_{1:t-1}\right) \\
= \mcn \left( \seriestep_t;\mu_{\paramdec}\left(\series_{1:t}^{\latent_{1:t}},\series_{1:t-1}\right), \sigma_\paramdec\left(\series_{1:t}^{\latent_{1:t}},\series_{1:t-1}\right) \right) \eqsp, 
\end{equation*}
with the notation $\series_{1:t}^{\latent_{1:t}} = (\seriestep_{1}^{z_{1}},\ldots,\seriestep_{t}^{z_{t}})$, where $\mu_\paramdec$ and $\sigma_\paramdec$ are parametric functions (such as recurrent or attention-based neural networks), and $\mcn(\cdot;\mu,\sigma)$ is the Gaussian probability density function with mean $\mu$ and variance $\sigma^2 I_d$.  
Such discrete clustering enables the model to capture heterogeneous behaviours within the reference dataset and to generalise better across different time series by leveraging learned groupings. It also facilitates transfer learning and adaptation, especially in multi-series contexts where the available data for any given single series might be limited. The simplest choice for the decoder is $\mu_{\paramdec}(\series_{1:t}^{\latent_{1:t}},\series_{1:t-1}) = x_t^{\latentpt}$, \emph{i.e.}, where the predicted observations are centred on a specific value of the reference dataset at that time step. 

\emph{Encoding.} In order to introduce a flexible variational distribution that can adapt to different tasks or data points with limited supervision, few-shot-based variational learning can be employed. The proposed variational distribution can be fine-tuned or conditioned on a small subset of the dataset, thereby enabling the construction of task-adaptive posteriors. 
For a sequence $\series_{1:T}$, our aim is to highlight the reconstruction bound proposed in Proposition~\ref{prop:main:markov}; in particular, its dependency with respect to the length $T$ of the sequences. Therefore, we introduce a simple Markovian variational distribution satisfying the decomposition proposed in \Cref{prop:main:markov}.

Denote $\horizon \in \N$ the horizon (number of time steps which the model predicts at a time step) and $\lookback \in \N$ the lookback (number of time steps input at a time step), and consider a generalisation of the model where we reconstruct $\horizon$ time steps at once, rather than a single step.
Rather than learning the variational approximation of $\decdistrib_\paramdec\left(\latent_{t:t+\horizon}\mid \latent_{1:t-1}, \series_{1:t-1}\right)$, a lightweight, structured approach is used, which allows a weighted contribution of both the locally and globally relevant time series. This approach follows the empirical findings of \citet{tonekaboni_decoupling_2022}, who learn two encoders; and of \citet{lee2025lightweight}, who combine a foundation model with a small forecaster. The first paper uses a local encoder to model the evolution of the series over time and a global encoder to model time-independent characteristics. The second provides an online linear forecaster which complements a foundation model's forecast, with the forecasts of the two models combined linearly at inference. Taking inspiration from this, our approach does not need to train two encoders nor deploy an online model alongside a foundation model, but uses distance calculations to locate the most similar time series in the pool of $B$ similar samples $\{\series_{\timerange}^{b}\}_{1\leq b\leq B}$. For each horizon window $\{(1, \ldots, \horizon), (\horizon+1, \ldots, 2\horizon), \ldots\}$, the distribution $\encdistrib_\paramenc\left(\latent_{t:t+\horizon}\mid \latent_{1:t-1}, \series_{1:t-1}\right)$ is built using:
\begin{itemize}
    \item a pretrained model $m_\theta$ that predicts $\mathbf{x}_{t:t+\horizon}$;
    \item a distance $\dist$ between the time series generated so far and each of the $B$ series in the reference pool. This can be interpreted as the global distance between the reference time series and the generated trajectory, and be observed over the whole series or only a lookback window $\lookback$;
    \item a distance $\dist_\textrm{loc}$ to measure dissimilarities between the prediction provided by $m_\theta$ and the time series in the reference dataset $\dataset$. This can be read as the local distance between the next steps in the time series and the generated possible trajectory;
    \item user-selected hyperparameters $\beta, \gamma \in \R^+$ which weigh the contribution of the local and global distance.
\end{itemize}
This distribution is given by:
\begin{equation*}
    \encdistrib_\paramenc\left(\latent_{\step:\step+\horizon} = \mathbf{\samplesindex}\mid \latent_{\step-\lookback:\step-1}, \series_{\step-\lookback:\step-1}\right) \propto \exp\left(-\beta \rmd\left(\serie_{\step-\lookback:\step-1}, \series_{\step-\lookback:\step-1}^{\latent_{\step-\lookback:\step-1}}\right)-\gamma \dist_\textrm{loc}\left(\hat{\series}_{\step:\step+\horizon}, \series_{\step:\step+\horizon}^\mathbf{\samplesindex}\right)\right) 
    \eqsp ,
\end{equation*}
where $\series^{\latent_{s+1:t}}_{s+1:t} = (x^{z_{s+1}}_{s+1},\ldots,x^{z_{t}}_{t})$, $\hat{\series}_{\step:\step+\horizon} = m_\theta(\series_{\step-\lookback:\step-1}^{\latent_{\step-\lookback:\step-1}})$ and $\mathbf{\samplesindex} = (i_0, i_1, \ldots, i_\horizon)$, where for \(i_h \in \{i_h\}_{h=0}^\horizon\), \(i_h\in\{1,\ldots, B\}\).

\Cref{algo:sample_hybrid} summarises the resulting procedure for reconstructing a time series. This method can use the pre-trained model $m_\theta$ only, the distances only, or combine a pre-trained model to guide sampling with the distances:  the latter two depend on the hyperparameters $\beta$ and $\gamma$ chosen.

\begin{algorithm}[ht]
	\caption{Structured Time Series Forecasting Framework.}
    \label{algo:sample_hybrid}
	\begin{algorithmic}[1]
        \Input Reference series $\{\mathbf{x}^b_{1:T}\}_{b=1}^B$, pre-trained model $m_\theta$, hyperparameters $\beta, \gamma$.
        \For{each step $\step$ in $\timerange$}
            \State Sample 
            \[
                \latent_{\step:\step+\horizon} \sim \encdistrib_\paramenc\left(\latent_{\step:\step+\horizon}=\mathbf{\samplesindex}\mid \latent_{\step-\lookback:\step-1}, \series_{\step-\lookback:\step-1}^\samplesindex\right) \propto \exp\left(-\beta \rmd\left(\serie_{\step-\lookback:\step-1}, \series_{\step-\lookback:\step-1}^{\latent_{\step-\lookback:\step-1}}\right)-\gamma \dist_\textrm{loc}\left(\hat{\series}_{\step:\step+\horizon}, \series_{\step:\step+\horizon}^\mathbf{\samplesindex}\right)\right) \eqsp .
            \]
            \State Sample \( \series_{\step:\step+\horizon}^{\latent_{\step:\step+\horizon}} \sim \decdistrib_\paramdec\left(\series_{\step:\step+\horizon}^{\latent_{\step:\step+\horizon}} = \series_{\step:\step+\horizon}^\samplesindex \mid \latent_{\step:\step+\horizon}, \serie_{\timerange}\right) \).
        \EndFor
    \Output Reconstructed series $\hat{\mathbf{x}}_{1:T}$.
    \end{algorithmic}
\end{algorithm}

\section{Conclusion}
\label{sec:conclusion}
We introduced a PAC-Bayesian reconstruction bound for VAEs with Markovian latent structure, providing theoretical guarantees for sequential latent variable models beyond the i.i.d. setting. In particular, our analysis shows that the resulting generalisation bound can remain stable with respect to the trajectory length, shedding light on the interplay between temporal depth and sample size in time-series forecasting.

Our work relies on assumptions, such as a control on the bias of the variational distribution, or the reconstruction losses needing to be sub-gamma. \Cref{prop:conditional_hoeffding} assumes $\theta\in\Theta$ to be fixed independently of observations, which in practise would mean using a decoder trained independently of the training data. Whilst this may seem to be a restrictive assumption, it is similar to the assumption from \citet{mbacke_statistical_2023}, which present bounds that work for a decoder, but uniformly for all encoders, and therefore requires to train the VAE on different samples from the ones used to calculate the bounds. These limitations may be restrictive in practice; building on existing literature which relaxes these conditions is an important direction for future research. We provide a framework which can serve as a paradigm where the assumptions are verified, and can serve as a starting point for future illustration of the reconstruction bounds derived.

Overall, this work provides a first step toward a theoretically grounded understanding of VAE-based time series models, and opens the door to integrating generalisation guarantees more tightly into the design of sequential deep learning systems.

\bibliography{references}
\bibliographystyle{tmlr}

\subsubsection*{Broader Impact Statement}
This work advances theoretical understanding of a broadly used machine learning model. The statistical findings aim to advance the theoretical machine learning field and could have a broad set of impacts, none of which the authors feel should be explicitly mentioned here.

\appendix
\section{Proofs of Theoretical Results}
\label{app:proofs}
We first rework for completeness \citet[Lemma B.1.]{mbacke_statistical_2023} written in the context of time series. 

\begin{lemma}
\label{lemma_conditional_hoeffding}
    Assume $\datagendis$ to be the data-generating distribution. Let $\latentspacept \coloneqq \left\{ \latent_{\timerange}^\samplesindex \right\}_{\samplesindex = 1}^\datasamples$. Then for any loss function $\ell:
    \R^{\datadim \times d_{\ell}} \times \R^{\datadim\times d} \rightarrow \R^+$, prior $p$, variational distribution $q$, real number $\lambda>0$ and for all i.i.d. sequences $(\series_{\timerange}^\samplesindex)_{\samplesindex=1}^\datasamples$,
    \begin{multline}
    \label{inequ_bound_hoeffding}
        \frac{1}{\datasamples}\sum_{\samplesindex=1}^\datasamples \E_{
        \encdistrib\left(\cdot \mid \data_{\timerange}^\samplesindex\right)} \biggl[   \E_{
        \datagendis}\bigl[ \ell(\latent_{\timerange}, \data_{\timerange}) \bigr]\biggr]  
        - \frac{1}{\datasamples}\sum_{\samplesindex=1}^\datasamples \E_{
        \encdistrib\left(\cdot \mid \data_{\timerange}^\samplesindex\right)} \biggl[ \ell(\latent_{\timerange}, \data_{\timerange}^\samplesindex)\biggr]  \\
         \leq \frac{1}{\lambda}\sum_{\samplesindex=1}^\datasamples\kl\bigl( \encdistrib(\cdot\mid \data_{\timerange}^\samplesindex) \vert \vert \decdistrib\left(\cdot\right)\bigr) + \frac{1}{\lambda}\log \E_{\latentspacept\sim p^{\otimes\datasamples}\left(\cdot \right)}\biggl[ \exp\biggl( \frac{\lambda}{\datasamples}\cdot \sum_{\samplesindex = 1}^{\datasamples}  \bigl\{\E_{
         \datagendis}\bigl[\ell\bigl(\latent_{\timerange}^\samplesindex, \data_{\timerange}\bigr)\bigr] - \ell\bigl(\latent_{\timerange}^\samplesindex, \data_{\timerange}^\samplesindex\bigr) \bigr\} \biggr) \biggr] \eqsp.
    \end{multline}
\end{lemma}
\begin{proof}
Let $(\data_{\timerange}^\samplesindex)_{1\leq \samplesindex \leq \datasamples}$ be sampled i.i.d. from distribution $\datagendis$. Notice from the properties of the exponential function and the independence of the samples that:
    \begin{multline*}
        \E_{\latentspacept \sim p^{\otimes\datasamples}}\biggl[ \exp \biggl( \frac{\lambda}{\datasamples} \cdot\sum_{\samplesindex = 1}^{\datasamples}  \bigl\{\E_{
        \datagendis}\bigl[\ell\bigl(\latent_{\timerange}^\samplesindex, \data_{\timerange}\bigr)\bigr] - \ell\bigl(\latent_{\timerange}^\samplesindex, \data_{\timerange}^\samplesindex\bigr) \bigr\} \biggr) \biggr] \\
        = \prod_{\samplesindex = 1}^{\datasamples}\E_{\latent_{\timerange} \sim p}\biggl[  \exp\biggl(\frac{\lambda}{\datasamples} \bigl\{\E_{
        \datagendis}\bigl[\ell\bigl(\latent_{\timerange}, \data_{\timerange}\bigr)\bigr] - \ell\bigl(\latent_{\timerange}, \data_{\timerange}^\samplesindex\bigr) \bigr\} \biggr) \biggr] \eqsp.
    \end{multline*}
    Then, by the Donsker-Varadhan change of measure \citep{donsker_asymptotic_1976}:
    \begin{multline*}
        \E_{
        \decdistrib}
        \biggl[ \exp\biggl(\frac{\lambda}{\datasamples} \bigl\{\E_{
        \datagendis}\bigl[\ell\bigl(\latent_{\timerange}, \data_{\timerange}\bigr)\bigr] - \ell\bigl(\latent_{\timerange}, \data_{\timerange}^\samplesindex\bigr) \bigr\} \biggr) \biggr]  \\
        \geq   \exp \biggl( \E_{
        \encdistrib\left(\cdot \mid \data_{\timerange}^\samplesindex\right)}\biggl[ \frac{\lambda}{\datasamples} \bigl\{ \E_{
        \datagendis}\bigl[ \ell(\latent_{\timerange}, \data_{\timerange}) \bigr] - \ell(\latent_{\timerange}, \data_{\timerange}^\samplesindex) \bigr\}\biggr] - \kl\bigl( \encdistrib(\cdot\mid \data_{\timerange}^\samplesindex) \vert \vert \decdistrib \left(\cdot\right)\bigr)\biggr)\eqsp. 
    \end{multline*}
Therefore, since the logarithm function is monotonic and can be applied to both sides of the inequality,
\begin{multline*}
    \frac{\lambda}{\datasamples}\sum_{\samplesindex=1}^\datasamples \E_{
    \encdistrib\left(\cdot \mid \data_{\timerange}^\samplesindex\right)} 
    \left[  \E_{
    \datagendis}\bigl[ \ell(\latent_{\timerange}, \data_{\timerange}) \bigr] - \ell(\latent_{\timerange}, \data_{\timerange}^\samplesindex) \right] - \sum_{\samplesindex=1}^\datasamples\kl\bigl( \encdistrib(\cdot\mid \data_{\timerange}^\samplesindex) \vert \vert \decdistrib \left(\cdot\right) \bigr) \\
    \leq\log \E_{\latentspacept \sim p^{\otimes\datasamples}\left(\cdot\right)}\biggl[ \exp\biggl( \frac{\lambda}{\datasamples}\cdot \sum_{\samplesindex = 1}^{\datasamples}  \bigl\{\E_{
    \datagendis}\bigl[\ell\bigl(\latent_{\timerange}^\samplesindex, \data_{\timerange}\bigr)\bigr] - \ell\bigl(\latent_{\timerange}^\samplesindex, \data_{\timerange}^\samplesindex\bigr) \bigr\}\biggr) \biggr] \eqsp,
\end{multline*}
where the logarithm and exponential cancel each other on the LHS. 
With the linearity of the expectation, moving the sum of the KL divergences to the upper bound and dividing both sides by $\lambda>0$ completes the proof.
\end{proof}

\subsection{Proof of Proposition~\ref{prop:conditional_hoeffding}}
\label{subsec:proof:main}

\begin{proof}
For latent variables $\latentspacept=\{\latent_{\timerange}^\samplesindex\}_{\samplesindex=1}^{\datasamples}$, define
$$
G_\datasamples(\latentspacept,\dataset)=\sum_{\samplesindex=1}^{\datasamples}\left\{\E_{\data_{\timerange}\sim\datagendis}\left[\ell_\theta(\latent_{\timerange}^\samplesindex,\data_{\timerange})\right]-\ell_\theta(\latent_{\timerange}^\samplesindex,\data_{\timerange}^\samplesindex)\right\}\eqsp.
$$
For all $1\leq\samplesindex\leq\datasamples$, conditionally on $\latent_{\timerange}^\samplesindex$, the random variable $\ell_\theta(\latent_{\timerange}^\samplesindex,\data_{\timerange}^\samplesindex)$ takes its values in $[0,b]$. Therefore, Hoeffding's lemma yields
$$
\E_{\data_{\timerange}^\samplesindex\sim\datagendis}\left[\exp\left(\frac{\lambda_\datasamples}{\datasamples}\left\{\E_{\data_{\timerange}\sim\datagendis}\left[\ell_\theta(\latent_{\timerange}^\samplesindex,\data_{\timerange})\right]-\ell_\theta(\latent_{\timerange}^\samplesindex,\data_{\timerange}^\samplesindex)\right\}\right)\right]\leq\exp\left(\frac{\lambda_\datasamples^2b^2}{8\datasamples^2}\right)\eqsp.
$$
Since the observations are independent,
$$
\E_{\dataset\sim\datagendis^{\otimes\datasamples}}\left[\exp\left(\frac{\lambda_\datasamples}{\datasamples}G_\datasamples(\latentspacept,\dataset)\right)\right]\leq\exp\left(\frac{\lambda_\datasamples^2b^2}{8\datasamples}\right)\eqsp.
$$
Define
$$
\Gamma_\datasamples(\dataset)=\E_{\latentspacept\sim\decdistrib^{\otimes\datasamples}}\left[\exp\left(\frac{\lambda_\datasamples}{\datasamples}G_\datasamples(\latentspacept,\dataset)\right)\right]\eqsp.
$$
By Fubini's theorem,
$$
\E_{\dataset\sim\datagendis^{\otimes\datasamples}}\left[\Gamma_\datasamples(\dataset)\right]\leq\exp\left(\frac{\lambda_\datasamples^2b^2}{8\datasamples}\right)\eqsp.
$$
Therefore, by Markov's inequality, with probability at least $1-\eta_\datasamples$,
$$
\Gamma_\datasamples(\dataset)\leq\frac{1}{\eta_\datasamples}\exp\left(\frac{\lambda_\datasamples^2b^2}{8\datasamples}\right)\eqsp.
$$
On this event, Lemma~\ref{lemma_conditional_hoeffding} gives
\begin{multline*}
\frac{1}{\datasamples}\sum_{\samplesindex=1}^{\datasamples}\E_{\encdistrib(\cdot\mid\data_{\timerange}^\samplesindex)}\left[\E_{\data_{\timerange}\sim\datagendis}\left[\ell_\theta(\latent_{\timerange},\data_{\timerange})\right]\right]-\frac{1}{\datasamples}\sum_{\samplesindex=1}^{\datasamples}\E_{\encdistrib(\cdot\mid\data_{\timerange}^\samplesindex)}\left[\ell_\theta(\latent_{\timerange},\data_{\timerange}^\samplesindex)\right]\\
\leq\frac{1}{\lambda_\datasamples}\sum_{\samplesindex=1}^{\datasamples}\kl\left(\encdistrib(\cdot\mid\data_{\timerange}^\samplesindex)||\decdistrib\right)+\frac{1}{\lambda_\datasamples}\log\Gamma_\datasamples(\dataset)\eqsp.
\end{multline*}
Using the previous upper bound on $\Gamma_\datasamples(\dataset)$ concludes the proof.
\end{proof}

\subsection{Proof of Proposition~\ref{prop:main:markov}}
\label{subsec:pf_main}
Proposition~\ref{prop:main:markov} requires to control the bias of the variational approximation of expectations of additive state functionals between different sequences of observations. This control is provided in Proposition~\ref{prop:MC:df}, which establishes that this error grows at most linearly in the number of observations under H\ref{hyp:strongmixing}.

\begin{proposition}
    \label{prop:MC:df}
    Assume that H\ref{hyp:strongmixing} and H\ref{hyp:lip} hold. Then for all function $f\in\mathcal{E}$, $\data_{1:T}$, $\data_{1:T}'$, 
    $$
\left|\vd{1:T,x}[f] - \vd{1:T,x'}[f]\right|\leq \sum_{k=1}^T\sum_{s=1}^{T-k+1}\|C^\parvar_{T-s+1}\|_\infty\prod_{u=s+1}^{T-k+1}\left(1 - \frac{\underline{\sigma}^\parvar_{T-u+2}(\data_{1:T})}{\overline{\sigma}^\parvar_{T-u+2}(\data_{1:T})}\right)\mathrm{d}_{T-s+1}(\data_{1:T},\data'_{1:T})\|f_k\|_{\infty}\eqsp,
    $$
    where $\underline{\sigma}^\parvar_{k}(\data_{1:T})$ and $\overline{\sigma}^\parvar_{k}(\data_{1:T})$ are defined as in H\ref{hyp:strongmixing}.
\end{proposition}

\begin{proof}
    Let $f\in\mathcal{E}$. Then,
    \begin{equation*}
       \vd{1:T,x}[f] - \vd{1:T,x'}[f] =  \int  \vd{1:T,x}(\latent_{1:T}) f(\latent_{1:T})\rmd \latent_{1:T} - \int  \vd{1:T,x'}(\latent_{1:T}) f(\latent_{1:T})\rmd \latent_{1:T} 
       = \sum_{k=1}^T \Delta^{\varphi}_{k,x,x'}(f)\eqsp,
    \end{equation*}
    where
    \begin{multline*}
        \Delta^{\varphi}_{k,x,x'}(f) = \int  \vd{T,x}(z_T)\prod_{s=k+1}^{T}\vd{s-1\vert s,x}(z_{s},z_{s - 1}) f_k(z_{k})\rmd \latent_{k:T} \\ - \int  \vd{T,x'}(z_T)\prod_{s=k+1}^{T}\vd{s-1\vert s,x'}(z_{s},z_{s - 1}) f_k(z_{k})\rmd \latent_{k:T}\eqsp.
    \end{multline*}
    Then, for all $1\leq k\leq T$,
    $$
    \Delta^{\varphi}_{k,x,x'}(f) = \sum_{s=1}^{T-k+1} \delta^{\varphi}_{s,x,x'}(f)\eqsp,
    $$
    where
    \begin{multline*}
        \delta^{\varphi}_{s,x,x'}(f) = \int  \prod_{u=1}^{s-1}\vd{T-u+1\vert T-u+2,x'}(z_{T-u+2},z_{T-u+1}) \prod_{u=s}^{T-k+1}\vd{T-u+1\vert T-u+2,x}(z_{T-u+2},z_{T-u+1})  f_k(z_{k})\rmd \latent_{k:T}\\
        - \int  \prod_{u=1}^{s}\vd{T-u+1\vert T-u+2,x'}(z_{T-u+2},z_{T-u+1}) \prod_{u=s+1}^{T-k+1}\vd{T-u+1\vert T-u+2,x'}(z_{T-u+2},z_{T-u+1})  f_k(z_{k})\rmd \latent_{k:T}\eqsp,
    \end{multline*}
    where by convention $\vd{T\vert T+1,x'}(z_{T+1},z_{T})=\vd{T,x'}(z_{T})$, $\vd{T\vert T+1,x}(z_{T+1},z_{T})=\vd{T,x}(z_{T})$ and 
    \begin{equation*}
    \prod_{u=1}^{0}\vd{T-u+1\vert T-u+2,x'}(z_{T-u+2},z_{T-u+1}) = 1\eqsp,\quad\mathrm{and}\quad
    \prod_{u=T-k+2}^{T-k+1}\vd{T-u+1\vert T-u+2,x}(z_{T-u+2},z_{T-u+1})=1\eqsp.
    \end{equation*} 
    Define
    \begin{align*}
        \mu_{s}(\rmd z_{T-s+1}) &= \int \prod_{u=1}^{s-1}\vd{T-u+1\vert T-u+2,x'}(z_{T-u+2},z_{T-u+1})\vd{T-s+1\vert T-s+2,x}(z_{T-s+2},z_{T-s+1})\rmd \latent_{T-s+2:T} \eqsp,\\
        \tilde{\mu}_{s}(\rmd z_{T-s+1}) &= \int  \prod_{u=1}^{s}\vd{T-u+1\vert T-u+2,x'}(z_{T-u+2},z_{T-u+1})\rmd \latent_{T-s+2:T} \eqsp,
    \end{align*}
    which yields
    \begin{multline*}
       \delta^{\varphi}_{s,x,x'}(f) = \int  \mu_{s}(\rmd z_{T-s+1}) \prod_{u=s+1}^{T-k+1}\vd{T-u+1\vert T-u+2,x}(z_{T-u+2},z_{T-u+1})  f_k(z_{k})\rmd \latent_{T-s:k}
        \\- \int  \tilde{\mu}_{s}(\rmd z_{T-s+1}) \prod_{u=s+1}^{T-k+1}\vd{T-u+1\vert T-u+2,x}(z_{T-u+2},z_{T-u+1})  f_k(z_{k})\rmd \latent_{T-s:k}\eqsp. 
    \end{multline*}
    By assumption H\ref{hyp:strongmixing}, using standard results for uniformly minorised Markov chains, see for instance \cite[Lemma~4.3.13]{cappe2005inference}, the Dobrushin coefficient of each backward kernel $\vd{T-u+1\vert T-u+2,x}$ is bounded by $\rho^\parvar_{T-u+2}(\data_{1:T}) = 1 - \underline{\sigma}^\parvar_{T-u+2}(\data_{1:T})/\overline{\sigma}^\parvar_{T-u+2}(\data_{1:T})$. 
    For $s=1$, for all bounded and measurable function $\psi$,
    \begin{equation*}
         \left|\vd{T,x}[\psi]-\vd{T,x'}[\psi]\right|\\
         \leq C^\parvar_{T}\mathrm{d}_T(\data_{1:T},\data_{1:T}')\|\psi\|_\infty\eqsp.
    \end{equation*}
    In addition, for $s>1$, for all bounded and measurable function $\psi$,
    \begin{equation*}
         \left|\vd{T-s+1\vert T-s+2,x}(z_{T-s+2},\psi)-\vd{T-s+1\vert T-s+2,x'}(z_{T-s+2},\psi)\right|
         \leq C^\parvar_{T-s+1}(z_{T-s+2})\mathrm{d}_{T-s+1}(\data_{1:T},\data'_{1:T})\|\psi\|_\infty\eqsp.
    \end{equation*}
Therefore,    
    \begin{align*}
         \left|\mu_{s}[\psi] - \tilde{\mu}_{s}[\psi]\right|  \leq \|C^\parvar_{T-s+1}\|_\infty\mathrm{d}_{T-s+1}(\data_{1:T},\data'_{1:T})\|\psi\|_\infty\eqsp.
    \end{align*}
    
    Therefore, we have  
    $$
    \left|\delta^{\varphi}_{s,x,x'}(f)\right|\leq \|C^\parvar_{T-s+1}\|_\infty\left(\prod_{u=s+1}^{T-k+1}\rho^\parvar_{T-u+2}(\data_{1:T})\right)\|f_k\|_{\infty}\mathrm{d}_{T-s+1}(\data_{1:T},\data'_{1:T})\eqsp,
    $$
    which concludes the proof.
\end{proof}

\begin{corollary}
    \label{cor:lip}
    Assume that H\ref{hyp:strongmixing} and H\ref{hyp:lip} hold.  Assume that  there exist $M$, $\underline{\sigma}^\parvar$, $\overline{\sigma}^\parvar$, $C^\parvar$,  such that for all $1\leq k \leq T$, $\|f_k\|_\infty\leq M$, $\|C^\parvar_{k}\|_\infty\leq C$, and for all $1\leq k \leq T$, $\overline{\sigma}^\parvar_{k}(\data_{1:T}) \leq \overline{\sigma}^\parvar $ and $\underline{\sigma}^\parvar_{k}(\data_{1:T})\geq \underline{\sigma}^\parvar$, then for all function $f\in\mathcal{E}$, $\data_{1:T}$, $\data_{1:T}'$,
    $$
    \left|\vd{1:T,x}[f] - \vd{1:T,x'}[f]\right|\leq  \frac{CM}{1-\rho^\parvar }\sum_{s=1}^{T} \mathrm{d}_{T-s+1}(\data_{1:T},\data'_{1:T})\eqsp,
    $$
    where $\rho^\parvar = 1 - \underline{\sigma}^\parvar/\overline{\sigma}^\parvar$.
\end{corollary}

\paragraph{Proof of \Cref{prop:main:markov}.}
\begin{proof}
    By \eqref{eq:backward:factorization} that for all $\data_{1:T}\in\R^{T\times d}$, $\vd{1:T,x}(\latent_{1:T})=  \vd{T,x}(z_T)\prod_{k=2}^{T}\vd{k-1\vert k,x}(z_{k},z_{k - 1})$. \Cref{cor:lip} states that there exists $C, M>0$ and $\rho^\parvar$ such that:
    \begin{equation}
    \label{eq:gap_linear_bound}
        \left|\vd{1:T,x}[f] - \vd{1:T,x'}[f]\right|\leq  \frac{CM}{1-\rho^\parvar }\sum_{s=1}^{T} \mathrm{d}_{s}(\data_{1:T},\data'_{1:T})\eqsp.
    \end{equation}
    Using $f(\latent_{\timerange}) = \ell_\paramdec(\latent_{\timerange}, \data_{\timerange}) = \sum_{t=1}^T\ell_t(z_{\paramdec,t}, \datapt)/T$ and  \eqref{eq:gap_linear_bound} yields 
    \begin{equation*}
        \frac{1}{\datasamples}\sum_{\samplesindex = 1}^\datasamples \vd{1:\datadim, x}[f] - \frac{1}{\datasamples}\sum_{\samplesindex = 1}^\datasamples \vd{1:\datadim, x^\samplesindex}[f] \leq \frac{CM}{\datasamples\datadim}  \frac{1}{1-\rho^\parvar }\sum_{i=
        1}^n\sum_{s=1}^{T} \mathrm{d}_{s}(\data_{1:T},\data^i_{1:T}) \eqsp,
    \end{equation*}
    so that
    \begin{equation*}
        \vd{1:\datadim, x}[f] - \frac{1}{\datasamples}\sum_{\samplesindex = 1}^\datasamples \vd{1:\datadim, x^\samplesindex}[f] \leq \frac{CM}{\datasamples\datadim}\frac{1}{1-\rho^\parvar }\sum_{i=1}^{n}\sum_{s=1}^{T} \mathrm{d}_{s}(\data_{1:T},\data^i_{1:T}) \eqsp .
    \end{equation*}
   Therefore,
    \begin{equation*}
        \vd{1:\datadim, x}[f] -  \frac{1 }{1-\rho^\parvar }\frac{CM}{\datasamples\datadim}\sum_{i=1}^{n}\sum_{s=1}^{T} \mathrm{d}_{s}(\data_{1:T},\data^i_{1:T}) \leq \frac{1}{\datasamples}\sum_{\samplesindex = 1}^\datasamples \vd{1:\datadim, x^\samplesindex}[f]
    \end{equation*}
    and
    \begin{equation*}
        \E_{\latent_{\timerange}\sim\encdistrib\left(\cdot \mid \data_{\timerange}\right)}\left[\ell_\paramdec\left(\latent_{\timerange}, \data_{\timerange}\right)\right] -  \frac{1}{1-\rho^\parvar }\frac{CM}{n\datadim}\sum_{i=1}^{n}\sum_{s=1}^{T} \mathrm{d}_{s}(\data_{1:T},\data^i_{1:T}) \leq \frac{1}{\datasamples}\sum_{\samplesindex = 1}^\datasamples \E_{\latent_{\timerange}\sim \encdistrib\left(\cdot \mid \data_{\timerange}^\samplesindex\right)}\left[\ell_\paramdec\left(\latent_{\timerange}, \data_{\timerange}\right)\right]\eqsp .
    \end{equation*}
    Notice that that since $\vd{\timerange, x}$ specifies the observations $x$ which are used to define the pdf, whereas $f$ is the function $f(z)$ evaluated for some given $x$ which is has encoded, the last term is taken under the expectation given $\data_{\timerange}^\samplesindex$ but is evaluated for the same value $\data_{\timerange}$ as the other pdf.
    Then,
    \begin{multline*}
        \E_{\data_{\timerange}\sim\datagendis}\left[\E_{\latent_{\timerange}\sim\encdistrib\left(\cdot \mid \data_{\timerange}\right)}\left[\ell_\paramdec\left(\latent_{\timerange}, \data_{\timerange}\right)\right] -  \frac{1}{1-\rho^\parvar }\frac{CM}{\datasamples\datadim}\sum_{i=1}^{n}\sum_{s=1}^{T} \mathrm{d}_{s}(\data_{1:T},\data^i_{1:T}) \right]\\
        \leq \E_{\data_{\timerange}\sim\datagendis}\left[\frac{1}{\datasamples}\sum_{\samplesindex = 1}^\datasamples \E_{\latent_{\timerange}\sim \encdistrib\left(\cdot \mid \data_{\timerange}^\samplesindex\right)}\left[\ell_\paramdec\left(\latent_{\timerange}, \data_{\timerange}\right)\right]\right]\eqsp .
    \end{multline*}
    Therefore,
    \begin{multline}
    \label{inequ:lb_for_lb}
        \E_{\data_{\timerange}\sim\datagendis}\left[\E_{\latent_{\timerange}\sim\encdistrib\left(\cdot \mid \data_{\timerange}\right)}\left[\ell_\paramdec\left(\latent_{\timerange}, \data_{\timerange}\right)\right] \right] - \frac{CM}{\datasamples\datadim}\frac{1}{1-\rho^\parvar }\sum_{i=1}^{n}\sum_{s=1}^{T} \E_{\data_{\timerange}\sim\datagendis}\left[\mathrm{d}_{s}(\data_{1:T},\data^i_{1:T})\right]\\
        \leq \frac{1}{\datasamples}\sum_{\samplesindex = 1}^\datasamples \E_{\latent_{\timerange}\sim \encdistrib\left(\cdot \mid \data_{\timerange}^\samplesindex\right)}\left[\E_{\data_{\timerange}\sim\datagendis}\left[\ell_\paramdec\left(\latent_{\timerange}, \data_{\timerange}\right)\right]\right]\eqsp .
    \end{multline}
By Proposition~\ref{prop:conditional_hoeffding}, with probability at least $ 1-\eta_\datasamples$,
\begin{multline*}
\frac{1}{\datasamples}\sum_{\samplesindex=1}^{\datasamples}\E_{\encdistrib_\paramenc(\cdot\mid\data_{\timerange}^\samplesindex)}\left[\E_{\data_{\timerange}\sim\datagendis}\left[\ell_\theta(\latent_{\timerange},\data_{\timerange})\right]\right]-\frac{1}{\datasamples}\sum_{\samplesindex=1}^{\datasamples}\E_{\encdistrib_\paramenc(\cdot\mid\data_{\timerange}^\samplesindex)}\left[\ell_\theta(\latent_{\timerange},\data_{\timerange}^\samplesindex)\right]\\
\leq \frac{1}{\lambda_\datasamples}\sum_{\samplesindex=1}^{\datasamples}\kl\left(\encdistrib_\paramenc(\cdot\mid\data_{\timerange}^\samplesindex)||\decdistrib\right) +\frac{\lambda_\datasamples b^2}{8\datasamples}+\frac{\log(1/\eta_\datasamples)}{\lambda_\datasamples}\eqsp.
\end{multline*}

    Lower bound the LHS of this inequality with \eqref{inequ:lb_for_lb}:
    \begin{multline*}
    \E_{\data_{\timerange}\sim\datagendis}\left[\E_{\latent_{\timerange}\sim\encdistrib\left(\cdot \mid \data_{\timerange}\right)}\left[\ell\left(\latent_{\timerange}, \data_{\timerange}\right)\right] \right] - \frac{CM}{\datasamples\datadim}\frac{1}{1-\rho^\parvar }\sum_{i=1}^{n}\sum_{s=1}^{T} \E_{\data_{\timerange}\sim\datagendis}\left[\mathrm{d}_{s}(\data_{1:T},\data^i_{1:T})\right] \\ \quad\quad\quad\quad\quad- \frac{1}{\datasamples}\sum_{\samplesindex=1}^\datasamples \E_{\latent_{\timerange}\sim\encdistrib\left(\cdot \mid\data_{\timerange}^\samplesindex\right)} \left[ \ell(\latent_{\timerange}, \data_{\timerange}^\samplesindex)\right]\\
        \leq \frac{1}{\datasamples}\sum_{\samplesindex=1}^\datasamples \E_{\latent_{\timerange}\sim\encdistrib\left(\cdot \mid\data_{\timerange}^\samplesindex\right)} \left[\E_{\data_{\timerange}\sim\datagendis}\left[ \ell(\latent_{\timerange}, \data_{\timerange}) \right]\right]  - \frac{1}{\datasamples}\sum_{\samplesindex=1}^\datasamples \E_{\latent_{\timerange}\sim\encdistrib\left(\cdot \mid\data_{\timerange}^\samplesindex\right)} \left[ \ell(\latent_{\timerange}, \data_{\timerange}^\samplesindex)\right] \eqsp. 
    \end{multline*}
    This gives:
    \begin{align*}
        \E_{\data_{\timerange}\sim\datagendis}&\left[\E_{\latent_{\timerange}\sim\encdistrib\left(\cdot \mid \data_{\timerange}\right)}\left[\ell\left(\latent_{\timerange}, \data_{\timerange}\right)\right] \right] - \frac{CM}{\datasamples\datadim}\frac{1}{1-\rho^\parvar }\sum_{i=1}^{n}\sum_{s=1}^{T} \E_{\data_{\timerange}\sim\datagendis}\left[\mathrm{d}_{s}(\data_{1:T},\data^i_{1:T})\right] \\
        &- \frac{1}{\datasamples}\sum_{\samplesindex=1}^\datasamples \E_{\latent_{\timerange}\sim\encdistrib\left(\cdot \mid\data_{\timerange}^\samplesindex\right)} \left[ \ell(\latent_{\timerange}, \data_{\timerange}^\samplesindex)\right] \leq \frac{1}{\lambda_n}\sum_{\samplesindex=1}^\datasamples\kl\left( \encdistrib\left(\cdot \mid\data_{\timerange}^\samplesindex\right) \vert \vert \decdistrib\left(\cdot\right)\right)  + \frac{\lambda_\datasamples b^2}{8\datasamples}+\frac{\log(1/\eta_\datasamples)}{\lambda_\datasamples} \eqsp,
    \end{align*}
    where moving the second term of the LHS to the upper bound completes the proof.
\end{proof}

\subsection{Proof of Proposition~\ref{prop:conditional}}
\label{subsec:pf_subgamma}

\begin{proof}
Define
$$
G_\datasamples(\latentspacept,\dataset)=\sum_{\samplesindex=1}^{\datasamples}\left\{\E_{\data_{\timerange}\sim\datagendis}\left[\ell_\theta(\latent_{\timerange}^\samplesindex,\data_{\timerange})\right]-\ell_\theta(\latent_{\timerange}^\samplesindex,\data_{\timerange}^\samplesindex)\right\}\eqsp.
$$
By independence of the observations and H\ref{hyp:subgamma}, applied with $\alpha=\lambda_\datasamples/\datasamples$,
\begin{align*}
\log\E_{\dataset\sim\datagendis^{\otimes\datasamples}}\left[\exp\left(\frac{\lambda_\datasamples}{\datasamples}G_\datasamples(\latentspacept,\dataset)\right)\right]
&\leq\frac{\lambda_\datasamples^2v}{2\datasamples(1-c\lambda_\datasamples/\datasamples)}\eqsp.
\end{align*}
Then, writing
$$
\Gamma_\datasamples(\dataset)=\E_{\latentspacept\sim\decdistrib^{\otimes\datasamples}}\left[\exp\left(\frac{\lambda_\datasamples}{\datasamples}G_\datasamples(\latentspacept,\dataset)\right)\right]\eqsp.
$$
yields
$$
\E_{\dataset\sim\datagendis^{\otimes\datasamples}}\left[\Gamma_\datasamples(\dataset)\right]\leq\exp\left(\frac{\lambda_\datasamples^2v}{2\datasamples(1-c\lambda_\datasamples/\datasamples)}\right)\eqsp.
$$
Therefore, by Markov's inequality, with probability at least $1-\eta_\datasamples$,
$$
\Gamma_\datasamples(\dataset)\leq\frac{1}{\eta_\datasamples}\exp\left(\frac{\lambda_\datasamples^2v}{2\datasamples(1-c\lambda_\datasamples/\datasamples)}\right)\eqsp.
$$
On this event, by the Donsker-Varadhan change-of-measure inequality,
\begin{equation*}
\frac{1}{\datasamples}\sum_{\samplesindex=1}^{\datasamples}\E_{\encdistrib(\cdot\mid\data_{\timerange}^\samplesindex)}\left[\E_{\datagendis}\left[\ell_\theta(\latent_{\timerange},\data_{\timerange})\right]-\ell_\theta(\latent_{\timerange},\data_{\timerange}^\samplesindex)\right]\
\leq\frac{1}{\lambda_\datasamples}\sum_{\samplesindex=1}^{\datasamples}\kl\left(\encdistrib(\cdot\mid\data_{\timerange}^\samplesindex)||\decdistrib\right)+\frac{1}{\lambda_\datasamples}\log\Gamma_\datasamples(\dataset)\eqsp.
\end{equation*}
Using the previous upper bound on $\Gamma_\datasamples(\dataset)$ concludes the proof.
\end{proof}

\end{document}